\documentclass[confcompsoc]{IEEEtran}
\IEEEoverridecommandlockouts
 
\usepackage{array}
\usepackage{subfigure}
\usepackage{stfloats}

\usepackage[
pdfa,
hidelinks,
pdftex, 
pdfdisplaydoctitle,
pdfpagelabels,
pdfauthor={Christopher Morris, Nils M. Kriege, Kristian Kersting, Petra Mutzel},
pdftitle={Faster Kernels for Graphs with Continuous Attributes via Hashing},  
pdfsubject={},
pdfkeywords={Data Mining, Big Data, Machine Learning, Graph Kernel, Locality Sensitive Hashing, ICDM},
pdfproducer={Latex with the hyperref package},
pdfcreator={pdflatex}
]{hyperref}

\usepackage{cmap}
\usepackage[T1]{fontenc}
\usepackage[utf8]{inputenc}

\usepackage{exscale}
\usepackage[l2tabu, orthodox]{nag}
\usepackage{etex}
\reserveinserts{28}	

\usepackage{amsmath}
\usepackage{amssymb}
\usepackage{amsthm}
\usepackage{amsfonts}
\usepackage{thmtools}	

\usepackage[activate={true,nocompatibility},final,kerning=true,spacing=true,stretch=0, shrink=58]{microtype}

\usepackage[mathic=true]{mathtools}
\usepackage{fixmath}
\usepackage{complexity}
\usepackage{bbm}
\usepackage{nicefrac}
\usepackage{mleftright}

\newtheorem{definition}{Definition}

\theoremstyle{plain}
\newtheorem{theorem}{Theorem}
\newtheorem{proposition}{Proposition}

\newtheorem{corollary}{Corollary}

\usepackage{todonotes}
\usepackage[capitalise]{cleveref}   

\usepackage{todonotes}
\usepackage{algorithm}
\usepackage{algorithmic}

\usepackage[backend=bibtex,citestyle=numeric-comp,style=ieee,natbib=false,doi=false,isbn=false,url=false]{biblatex} 
\usepackage{enumitem}
\setlist[enumerate]{noitemsep, topsep=0.3\topsep}
\setlist[description]{noitemsep, topsep=0.3\topsep}
\setlist[itemize]{noitemsep, topsep=0.3\topsep}
\usepackage{booktabs}
\usepackage{multirow}
\newcommand{\ra}[1]{\renewcommand{\arraystretch}{#1}}

\newcommand{\new}[1]{\emph{#1}}

\newcommand{\bbR}{\ensuremath{\mathbb{R}}}

\newcommand{\bbN}{\ensuremath{\mathbb{N}}}

\newcommand{\tp}[0]{\top}

\newcommand{\hashgraph}{\ensuremath{\boldsymbol{h}}}

\newcommand{\Comment}[1]{\hfill$\triangleright$ \textit{#1}}

\usepackage{tikz}
\usepackage{textcomp}
\usepackage{hyperref}
\usepackage{lipsum}

\newcommand\copyrighttext{%
  \footnotesize \textcopyright 2016 IEEE. Personal use of this material is permitted.
  Permission from IEEE must be obtained for all other uses, in any current or future
  media, including reprinting/republishing this material for advertising or promotional
  purposes, creating new collective works, for resale or redistribution to servers or
  lists, or reuse of any copyrighted component of this work in other works.}
\newcommand\copyrightnotice{%
\begin{tikzpicture}[remember picture,overlay]
\node[anchor=south,yshift=10pt] at (current page.south) {\fbox{\parbox{\dimexpr\textwidth-\fboxsep-\fboxrule\relax}{\copyrighttext}}};
\end{tikzpicture}%
}

\begin{document}
\title{Faster Kernels for Graphs with Continuous Attributes via Hashing}
\author{\IEEEauthorblockN{Christopher Morris, \and Nils M.\,Kriege, \and Kristian Kersting, \and Petra Mutzel\\}
\IEEEauthorblockA{TU Dortmund University \\ \{christopher.morris, nils.kriege, kristian.kersting, petra.mutzel\}@tu-dortmund.de}
}

\maketitle
\copyrightnotice
\begin{abstract}
While state-of-the-art kernels for graphs with discrete labels scale well to graphs with thousands of nodes, the few existing kernels for graphs with continuous attributes, unfortunately, do not scale well. To overcome this limitation, we present \emph{hash graph kernels}, a general framework to derive kernels for graphs with continuous attributes from discrete ones. The idea is to iteratively turn continuous attributes into discrete labels using randomized hash functions. We illustrate hash graph kernels for the Weisfeiler-Lehman subtree kernel and for the shortest-path kernel. The resulting
novel graph kernels are shown to be, both, able to handle graphs with continuous 
attributes and scalable to large graphs and data sets. 
This is supported by our theoretical analysis and demonstrated by an extensive 
experimental evaluation.
\end{abstract}
\IEEEpeerreviewmaketitle

\section{Introduction}
In several domains like chemo- and bioinformatics as well as social network and 
image analysis structured objects appear naturally.
Graph kernels are a key concept for the application of kernel methods to structured
data and various approaches have been developed in recent years, see~\cites{Vis+2010}{Neu+2015},
 and references therein. The considered graphs can be distinguished in 
(i) graphs with discrete labels, e.g., molecular graphs, where nodes are 
annotated by the symbols of the atoms they represent, and 
(ii) attributed graphs with (multi-dimensional) real-valued labels in addition 
to discrete labels.
Attributed graphs often appear in domains like bioinformatics~\cite{Borgwardt2005a}
or image classification~\cite{Harchaoui2007}, where attributes may represent 
physical properties of protein secondary structure elements or RGB values of 
colors, respectively.
Taking the continuous information into account has been proven empirically to be beneficial 
in several applications, e.g., see~\cite{Neu+2015,Borgwardt2005a, Fer+2013, Harchaoui2007, Kri+2012,  Ors+2015}.

Kernels are equivalent to the inner product in an associated feature space, where
a feature map assigns the objects of the input space to a feature vector.
The various graph kernels proposed in recent years can be divided into approaches
that either compute feature maps (i) explicitly, or (ii) implicitly~\cite{Kri+2014}.
Explicit computation schemes have been shown to be scalable and allow the use of 
fast linear support vector classifiers, e.g.,~\cite{Joa+2006}, while implicit computation schemes are often slow. 
Alternatively, we may divide graph kernels according to their ability to handle 
annotations of nodes and edges. The proposed graph kernels are either
(i) restricted to discrete labels, or (ii) compare annotations like continuous
values by user-specified kernels. Typically kernels of the first category  
implicitly compare annotations of nodes and edges by the Dirac kernel,
which requires values to match exactly and is not adequate for continuous values.
The two classifications of graph kernels mentioned above largely 
coincide:
Graph kernels supporting complex annotations use implicit computation schemes 
and do not scale well. Whereas graphs with discrete labels can be compared 
efficiently by graph kernels based on explicit feature maps. 
This is what we make use of to develop a unifying treatment. But first, let 
us touch upon related work. 

\subsection{Previous work}
In recent years, various graph kernels have been proposed. In~\cite{Gaertner2003} and~\cite{Kashima2003} graph kernels were proposed
based on random walks, which count the number of walks two graphs have in 
common. Since then, random walk kernels have been studied intensively, e.g.,~\cite{Kang2012,Kri+2014,Mahe2004,Vis+2010}.
Kernels based on tree patterns were initially proposed in~\cite{Ramon2003}. These two approaches were originally applied
to graphs with discrete labels, but the method of implicit computation supports 
comparing attributes by user-specified kernel functions. 
Kernels based on shortest paths~\cite{Bor+2005} are computed by performing $1$-step walks on the transformed input graphs, where
edges are annotated with shortest-path lengths.
A drawback of the approaches mentioned above is their high computational cost. 

A different line in the development of graph kernels focused particularly on 
scalable graph kernels. These kernels are typically computed efficiently by 
explicit feature maps, but are severely limited to graphs with discrete labels.
Prominent examples are kernels based on subgraphs up to a fixed 
size, e.g.,~\cite{She+2009}, or specific subgraphs like cycles and 
trees~\cite{Horvath2004}. Other approaches of this category encode the 
neighborhood of every node by different techniques~\cite{Hido2009,Neu+2015,She+2011}.

Recently, several kernels specifically designed for gra\-phs with continuous
attributes were proposed~\cite{Fer+2013,Kri+2012,Ors+2015}, and their 
experimental evaluation confirms the importance of handling continuous attributes
adequately. 

Several articles on scalable kernels for graphs with discrete labels propose the 
adaption of their approach to graphs with continuous attributes as future 
work, e.g., see~\cite{She+2009,She+2011}. 
Yet, only little work in this direction has been reported, which is most likely 
due to the fact that this in general is a non-trivial task.
An immediate approach is to discretize continuous values by binning. A key problem 
of this method is that two values, which only differ marginally may still fall 
in different bins and are then considered non-matching. 
Still, promising experimental results of such approaches have been reported
for certain data sets, e.g.,~\cite{Neu+2015}.

\subsection{Our Contribution}
We introduce \emph{hash graph kernels} for graphs with continuous attributes.
This family of kernels is obtained by a generic method, which iteratively hashes
continuous attributes to discrete labels in order to apply a base kernel for 
graphs with discrete labels. This allows to construct a single combined feature vector for a graph from the individual feature vectors of each iteration.
The essence of this approach is:
	\begin{quote}\it
	 The hash graph kernel framework lifts every graph kernel that supports discrete labels to a kernel which can handle continuous attributes.
	 \end{quote}
We exemplify this for two established graph kernels:
\begin{itemize}
\item We obtain a variation of the Weisfeiler-Lehman subtree kernel, which implicitly employs a non-trivial kernel on the node and edge annotations 
and is suitable for continuous values.
\item Moreover, we derive a variant of the \emph{shortest-path kernel} which also supports continuous attributes while being efficiently computable
by explicit feature maps. 
\end{itemize}
For both kernels we provide a detailed \emph{theoretical analysis}. Moreover, the effectiveness of these kernels is demonstrated in an extensive experimental study
on real-world and synthetic data sets. The results show that hash graph kernels are orders of magnitude
faster than state-of-the-art kernels for attributed graphs without drop in 
classification accuracy.

\section{Notation}\label{pre}
An \new{(undirected) graph} $G$ is a pair $(V,E)$ with a \emph{finite} set of \new{nodes} $V$ and a set of \new{edges} $E \subseteq \{ \{u,v\} \subseteq V \mid u \neq v \}$. We denote the set of nodes and the set of edges of $G$ by $V(G)$ and $E(G)$, respectively. For ease of notation we denote the edge $\{u,v\}$ in $E(G)$ by $(u,v)$ or $(v,u)$. Moreover, $N(v)$ denotes the \new{neighborhood} of $v$ in $V(G)$, i.e., $N(v) = \{ v' \in V(G) \mid (v, v') \in E(G) \}$. An \new{attributed graph} is a graph $G$ endowed with an \new{attribute function} $a \colon V(G) \to \bbR^d$ for $d \geq 1$. We say that $a(v)$ is an \new{attribute} of $v$ for $v$ in $V(G)$. A \emph{labeled graph} is an attributed graph with an attribute function $l$, where the codomain of $l$ is restricted to a (finite) alphabet, e.g., a finite subset of the natural numbers. Analogously, we say that $l(v)$ is a \new{label} of $v$ in $V(G)$.

Let $\chi$ be a non-empty set and let 
$k\colon \chi \times \chi \to \mathbb{R}$ 
be a function. Then $k$ is a \new{kernel} on $\chi$ if there is a real Hilbert space $\mathcal{H}_k$ and a mapping $\phi \colon \chi \to \mathcal{H}_k$ such that $k(x,y) = \langle \phi(x), \phi(y) \rangle$ for $x$ and $y$ in $\chi$, where $\langle \cdot, \cdot \rangle$ denotes the inner product of $\mathcal{H}_k$. We call $\phi$ a \new{feature map}, and $\mathcal{H}_k$ a \new{feature space}. Let $\mathcal{G}$ be a non-empty set of (attributed) graphs, then a kernel $k\colon \mathcal{G} \times \mathcal{G} \to \bbR$ is called \new{graph kernel}.  
We denote by $k_\delta \colon \chi \times \chi \to \bbR$ the Dirac kernel with $k_\delta(x,y)=1$ if $x=y$, and $0$ otherwise.

\section{Hash Graph Kernels}\label{shgk}

In this section we introduce hash graph kernels. The main idea of hash graph kernels is to map attributes to labels using a \new{family of hash functions} and then apply a kernel for graphs with discrete labels.

Let $\mathcal{H} = \{h : \bbR^d \to \bbN \}$ be a family of hash functions and $G$ a 
graph with attribute function $a \colon V(G) \to \bbR^d$.
We can transform $(G,a)$ to a graph with discrete labels by mapping each 
attribute $a(v)$ to $h(a(v))$ with some function $h$ in $\mathcal{H}$. For short, we write 
$\hashgraph(G)$ for the labeled graph obtained by this procedure. The function $h$ is drawn at random from the family of hash functions 
$\mathcal{H}$. This procedure is repeated multiple times in order to lower the variance.
Thus, we obtain a sequence of discretely labeled graphs $(\hashgraph_i(G))_{i=1}^I$, where
$I$ is the number of iterations.
Hash graph kernels compare these sequences of labeled graphs by an arbitrary 
graph kernel for labeled graphs, which we refer to as \emph{discrete base graph 
	kernel}, e.g., the Weisfeiler-Lehman subtree or the shortest-path kernel.

\begin{definition}[Hash graph kernel]\label{hgk}
	Let $\mathcal{H}$ be a family of hash functions and $k_\text{b}$ a discrete base graph 
	kernel, then the \emph{hash graph kernel} for two attributed graphs $G$ and 
	$H$ is defined as
	\begin{equation*}
	k_\text{HGK}(G,H) = \frac{1}{I} \sum_{i=1}^I k_\text{b} ( \hashgraph_i(G), \hashgraph_i(H) ), 
	\end{equation*}
	where  $\hashgraph_i$ is obtained by choosing hash functions from $\mathcal{H}$.
\end{definition}

We will discuss hash functions, possible ways to choose them from $\mathcal{H}$ and 
how they relate to the global kernel value in~\cref{hash} and~\cref{hgk-app}. We proceed with
the algorithmic aspects of hash graph kernels.
It is desirable for efficiency to compute explicit feature maps for graph kernels.
We can obtain feature vectors for hash graph kernels under the assumption that 
the discrete base graph kernel can be computed by explicit feature maps.
This can be achieved by concatenating the feature vectors for each iteration and normalizing the combined feature maps by $\sqrt{1/I}$ according to the pseudocode in~\cref{rbk}. 

\begin{algorithm}
	\begin{algorithmic}[1]
		\STATE \textbf{Input:} An attributed graph $(G,a)$, a graph feature map $\phi_b$ of the discrete base kernel $k_b$, and a parameter $I$ in $\bbN_{>0}$.
		\STATE \textbf{Output: } A feature vector $\Phi(G)$ for $(G,a)$.
		\item[]

		\FOR{$i$ in $\{1, \dots, I\}$}
		\STATE $(G,l) \gets \hashgraph_i(G)$  \Comment{Hash attributes to labels}
		\STATE $\Phi(G) \leftarrow \Phi(G) \oplus \phi_b((G,l))$ \Comment{Concatenate vectors}
		\ENDFOR

		\item[]
		\STATE  \textbf{Return }$ \sqrt{\nicefrac{1}{I}} \cdot \Phi(G)$ \Comment{Normalize}
	\end{algorithmic}
	\caption{Explicit feature maps for hash graph kernels}\label{rbk}
\end{algorithm}
\subsection{Analysis}

Since hash graph kernels are a normalized sum over discrete base graph kernels
applied to a sequence of transformed input graphs, it is clear that we again 
obtain a valid kernel.

For the explicit computation of feature maps by \cref{rbk} we get the following 
bound on the running time.

\begin{proposition}[Running Time]
  \cref{rbk} computes the hash graph kernel feature map for a graph $G$ in time 
$
	\mathcal{O} (I \cdot (\mathsf{T}_{\text{H}}(G) + \mathsf{T}_{\phi}(G))),
$
where $\mathsf{T}_{\text{H}}(G)$ denotes the running time to evaluate the 
   hash functions for $G$ and $\mathsf{T}_{\phi}(G)$ the running time to compute 
   the graph feature map of the discrete base graph kernel for $G$. 
\end{proposition}
\begin{proof}
	Directly follows from~\cref{rbk}.
\end{proof}
Notice that when we fix the number of iterations and assume $\mathsf{T}_{\text{H}}(G) \leq \mathsf{T}_{\phi}(G)$, the hash graph kernel can be computed in the same asymptotic running time as the discrete base graph kernel. Moreover, notice that lines 4 to 5 in \cref{rbk} can be easily executed in parallel.

\subsection{Hash Functions}\label{hash}

In this section we discuss possible realizations of the hashing technique used
to obtain hash graph kernels according to~\cref{hgk}. 

The key idea is to choose a family of hash functions and draw hash functions $h_1$ 
and $h_2$ in each iteration such that $\Pr[h_1(x) = h_2(y)]$ is an adequate 
measure of similarity between attributes $x$ and $y$ in $\bbR^d$.
For the case that $h_1=h_2$ drawn at random, such families of hash functions 
have been proposed, e.g., see \cite{Rah+2008, Andoni2008, Dat+2004}.
Unfortunately, these results do not lift to kernels composed of 
products of base kernels. Thus they do not directly transfer to hash graph kernels,
where complex discrete base graph kernels are employed. 
For example, let $k_\Delta$ be the hat kernel on $\bbR$ and $h$ a hash function,
such that $k_\Delta(x,y) = \Pr[h(x) = h(y)]$, see~\cite{Rah+2008}. However, in general
$$k_\Delta(a,c) \cdot k_\Delta(b,d) \neq \Pr\left[h(a) = h(c) \wedge h(b) = h(d)\right]\;.$$

To overcome this issue, we introduce the following concept.

\begin{definition}\label{ikh_fam}
	Let $k \colon \chi \times \chi \to \bbR$ be a kernel and let $\mathcal{H} = \{ h \colon \chi \to \mathcal{S} \}$ for some set $\mathcal{S}$ be a family of hash functions. Then $\mathcal{H}$ is an \emph{independent $k$-hash family} if $\Pr[h_1(x) = h_2(y)] = k(x,y)$ where $h_1$ and $h_2$ are chosen independently and uniformly at random from $\mathcal{H}$.
\end{definition}

\section{Hash Graph Kernel Instances}\label{hgk-app}

In the following we prove that hash graph kernels approximate implicit variants of the shortest-path and the Weisfeiler-Lehman subtree kernel for attributed graphs.

\subsection{Shortest-path kernel}

We first describe the implicit shortest-path kernel which can handle attributes. Let $(G,a)$ be an attributed graph and let $d_{uv}$ denote the length of the shortest path between $u$ and $v$ in $V(G)$. The kernel is then defined as 
\begin{equation*}
k^{k_{\text{A}},k_{\text{d}}}_{\text{Imp-SP}}(G,H) = \sum_{\substack{(u,v) \in V(G)^2\\ u \neq v}} \sum_{\substack{(w,z) \in V(H)^2\\ w \neq z}} k((u,v),(w,z)),
\end{equation*}
where
\begin{align*}\label{ksp}
k((u,v),(w,z)) = &k_{\text{A}}(a(u), a(w)) \cdot k_{\text{A}}(a(v), a(z))\\
&\cdot k_{\text{d}}(d_{uv}, d_{wz})\;.
\end{align*}
Here $k_{\text{A}}$ is a kernel for comparing node labels or attributes and $k_{\text{d}}$ is a kernel to compare shortest-path distances, such that $k_{\text{d}}(d_{uv}, d_{wz}) = 0$ if  $d_{uv} = \infty$ or $d_{wz} = \infty$.

If we set $k_{\text{A}}$ and $k_{\text{d}}$ to the Dirac kernel, we can compute an explicit mapping $\phi_{\text{SP}}$ for the kernel $k^{k_{\text{A}},k_{\text{d}}}_{\text{Imp-SP}}$: Assume a labeled graph $(G,l)$, then each component of $\phi_{\text{SP}}(G)$ counts the number of occurrences of a triple of the form $(l(u),l(v), d_{uv})$ for $(u,v)$ in $V(G)^2$, $u \neq v$, and  $d_{uv} < \infty$. It is easy to see that 
\begin{equation}\label{speq}
\phi_{\text{SP}}(G)^\tp \phi_{\text{SP}}(H) = k^{k_\delta, k_\delta}_{\text{Imp-SP}}(G,H)\;.
\end{equation}

The following theorem shows that the hash graph kernel approximates $k^{k_{\text{A}},k_\delta}_{\text{Imp-SP}}$ arbitrarily close by using the explicit shortest-path kernel as a discrete base kernel and an independent $k_{\text{A}}$-hash family.

\begin{theorem}[Approximation of implicit shortest-path kernel for continuous attributes]\label{spa}
	Let $k_{\text{A}} \colon \bbR^n \times \bbR^n \to \bbR$ be a kernel and let $\mathcal{H}$ be an independent $k_{\text{A}}$-hash family. Assume that in each iteration of~\cref{rbk} each attribute is mapped to a label using a hash function chosen independently and uniformly at random from $\mathcal{H}$. Then~\cref{rbk} with the explicit shortest-path kernel acting as the discrete base kernel approximates $k^{k_{\text{A}},k_\delta}_{\text{Imp-SP}}$ such that 
	\begin{equation*}
	\Pr\mleft[\mleft| \Phi(G)^\tp \Phi(H) - k^{k_{\text{A}},k_\delta}_{\text{Imp-SP}}(G,H) \mright| \geq \lambda \mright] \leq 2\exp \mleft(-2\lambda^2 I \mright)\;.
	\end{equation*}
	Moreover with any constant probability,
	\begin{equation*}
	\sup_{G,H \in \mathcal{G}} \left| \Phi(G)^\tp \Phi(H) - k^{k_{\text{A}},k_\delta}_{\text{Imp-SP}}(G,H) \right|\leq \epsilon,
	\end{equation*}
	for $\epsilon >0$.
\end{theorem}
\begin{proof}
	By assumption, we have $\Pr[h_1(a) = h_2(a')] = k_{\text{A}}(a,a')$
	for $h_1$ and $h_2$ chosen independently and uniformly at random from $\mathcal{H}$. 
	Since we are using a Dirac kernel to compare discrete attributes, we get $\mathbf{E}[k_{\delta}(h_1(a), h_2(a'))]= k_{\text{A}}(a,a')$.
	Since $\mathcal{H}$ is an independent $k_{\text{A}}$-hash family, $\Pr[h_1(a) = h_2(b) \wedge h_3(c) = h_4(d)] = k_{\text{A}}(a,b) \cdot k_{\text{A}}(c,d)$ for $h_1, h_2, h_3$, and $h_4$ chosen independently and uniformly at random from $\mathcal{H}$. Hence, by the above, $\mathbf{E}[k_{\delta}(h_1(a), h_2(b)) \cdot k_{\delta}(h_3(c), h_4(d))]= k_{\text{A}}(a,b) \cdot  k_{\text{A}}(c,d).$
	By~\cref{speq} and using the linearity of expectation,
$
	\mathbf{E}[ \Phi(G)^\tp \Phi(H)] = k_{\text{Imp-SP}}(G,H).
$
	Now assume that $k_{\text{Imp-SP}}(G,H)$ is normalized to $[0,1]$, then the first claim follows from the Hoeffding bound~\cite{Hoe1963}. In order to derive the second claim, we choose
$
	I \geq \frac{1}{2\epsilon^2} \log( |\mathcal{G}|^2\cdot c)
$
	with a large enough constant $c > 1$, where $I$ is the number of iterations in~\cref{rbk}. From the first claim, we get
	\begin{align*}
	&\Pr\mleft[\mleft| \Phi(G)^\tp \Phi(H)-\!k^{k_{\text{A}},k_\delta}_{\text{Imp-SP}}(G,H) \mright| > \epsilon \mright]\\
	&\leq 2\exp \mleft(- \log( |\mathcal{G}|^2\cdot c) \mright)\\
	& = \frac{1}{c/2 \cdot |\mathcal{G}|^2}\;.
	\end{align*}
	The claim then follows from the Union bound.
\end{proof}

Notice that we can also approximate $k^{k_{\text{A}},k_{\text{d}}}_{\text{Imp-SP}}$ by employing a $k_{\text{d}}$-independent hash family.

\subsection{Weisfeiler-Lehman subtree kernel}

By the same arguments, we can derive a similar result for the Weisfeiler-Lehman subtree kernel. The following Proposition derives an implicit version of the Weisfeiler-Lehman subtree kernel.

\begin{proposition}[Implicit Weisfeiler-Lehman subtree kernel] Let 
	\begin{align*}
	k^h_{\text{Imp-WL}}(G,H) = \displaystyle\sum_{i=0}^{h}\sum_{{\substack{v \in V(G),\\ \,v' \in V(H)}}} k_i(v,v'),  
	\end{align*}
	\text{where} 
	\begin{align*}
	&k_i(v,v') = 
	\begin{dcases}
	k_{\delta}(l(v),l(v')) &\!\!i = 0,\\
	k_{i-1}(v,v') \cdot f(v,v') &\!\!i > 0 \wedge \mathcal{M}_i(v, v') \neq \emptyset,\\
	0 &\!\!i > 0 \wedge \mathcal{M}(v,v') = \emptyset,
	\end{dcases}\\
		\end{align*}
\text{and}\\
	\begin{align*}
	&f(v,v') = {|M_i(v,v')|}^{-1}  \sum_{{R \in M_i(v,v')}}  \prod_{{(w,w') \in R}} k_{i-1}(w,w'),
\end{align*}
	where $\mathcal{M}_i(v,v')$ is the set of bijections $b\colon V(G) \to V(H)$ between $N(v)$ and 
	$N(v')$ such that $k_{i-1}(w,w') > 0$ for $b(w)=w'$. Then $k^h_{\text{Imp-WL}}$ is equivalent to the Weisfeiler-Lehman subtree kernel.
\end{proposition}
\begin{proof}
	Follows from~\cite[Theorem 8]{She+2011}. 
\end{proof} 
We show that~\cref{rbk} with the \emph{(explicit)} Weisfeiler-Lehman subtree kernel acting as the discrete base graph kernel probabilistically approximates the graph kernel $k^{h,k_\text{A}}_{\text{Imp-WL}}$, where $k^{h,k_{\text{A}}}_{\text{Imp-WL}}$ is defined by substituting $k_{\delta}$ in the definition of $k^h_{\text{Imp-WL}}(G,H)$ by a kernel $k_{\text{A}} \colon \bbR^n \times \bbR^n \to \bbR$.

\begin{corollary}[Approximation of implicit Weisfeiler-Leh\-man subtree kernel for continous attributes]
	Let $k_{\text{A}} \colon \bbR^n \times \bbR^n \to \bbR$ be a kernel and let $\mathcal{H}$ be an independent $k_{\text{A}}$-hash family. Assume that in each iteration of~\cref{rbk} each attribute is mapped to a label using a hash function chosen independently and uniformly at random from $\mathcal{H}$. Then~\cref{rbk} with the Weisfeiler-Lehman subtree kernel with $h$ iterations acting as the discrete base kernel approximates $k^{h, k_\text{A}}_{\text{Imp-WL}}$ such that 
	\begin{equation*}
	\Pr\mleft[\mleft| \Phi(G)^\tp \Phi(H) -  k^{h,k_\text{A}}_{\text{Imp-WL}}(G,H) \mright| \geq \lambda \mright] \leq 2\exp \mleft(-2\lambda^2 I \mright).
	\end{equation*}
	Moreover with any constant probability,
	\begin{equation*}
	\sup_{G,H \in \mathcal{G}} \left| \Phi(G)^\tp \Phi(H) - k^{h,k_\text{A}}_{\text{Imp-WL}}(G,H) \right|\leq \epsilon,
	\end{equation*}
	for $\epsilon$ > 0.
\end{corollary}
\begin{proof}
Since $k^h_{\text{Imp-WL}}$ is written as a sum of products and sums of Dirac kernels, we can again use the property of $k$-independent hash functions and argue analogously to the proof of~\cref{spa}.
\end{proof}

\section{Experimental Evaluation}\label{ex}
	Our intention here is to investigate the benefits of hash graph kernels compared to the state-of-the-art. More precisely, we address the following questions:
	
	\begin{enumerate}
		\item[\textbf{Q1}] How do hash graph kernels compare to state-of-the-art graph kernels for attributed graphs in terms of classification accuracy and running time?
		\item[\textbf{Q2}] How does the choice of the discrete base kernel influence the classification accuracy?
		\item[\textbf{Q3}] Does the number of iterations influence the classification accuracy of hash graph kernels in practice?
	\end{enumerate}
	
	\subsection{Data Sets and Graph Kernels} We used the following data sets to evaluate and compare hash graph kernels: \textsc{Enzymes}~\cite{Borgwardt2005a,Fer+2013}, \textsc{Frankenstein}~\cite{Ors+2015}, \textsc{Proteins}~\cite{Borgwardt2005a,Fer+2013}, \textsc{SyntheticNew}~\cite{Fer+2013}, and \textsc{Synthie}.\footnote{Due to space constraints we refer to \url{https://ls11-www.cs.uni-dortmund.de/staff/morris/graphkerneldatasets} for descriptions, references, and statistics.}
   The data set \textsc{Synthie} consists of 400 graphs, subdivided into four classes, with 15 real-valued node attributes. It was generated as follows: First, we generated two Erd\H{o}s-Rényi graphs $G_1$ and $G_2$ on $n$ vertices with edge probability $p=0.2$. From each we generated a seed set $\mathcal{S}_i$ for $i$ in $\{1,2\}$ of 200 graphs by randomly adding or deleting 25\% of the edges of $G_i$. Connected graphs were obtained by randomly sampling 10 seeds and randomly adding edges between them. We generated the class $\mathcal{C}_1$ of 200 graphs, choosing seeds with probability 0.8 from $\mathcal{S}_1$ and 0.2 from $\mathcal{S}_2$ and the class $\mathcal{C}_2$ with interchanged probabilities. Finally, we generated a set of real-valued vectors of dimension 15, subdivided into classes $\mathcal{A}$ and $\mathcal{B}$, following the approach of~\cite{Guy+2003}. 
   We then subdivided $\mathcal{C}_i$ into two classes $\mathcal{C}^A_i$ and $\mathcal{C}^B_i$ by drawing a random attribute for each node. For the class $\mathcal{C}^A_i$ ($\mathcal{C}^B_i$), a node stemming from  a seed in $\mathcal{S}_1$ ($\mathcal{S}_2$) was annotated by an attribute from $\mathcal{A}$, and from $\mathcal{B}$ otherwise.

	We implemented hash graph kernels with the explicit shor\-test-path graph kernel (\textsc{HGK-SP}) and the Weisfeiler-Leh\-man subtree kernel (\textsc{HGK-WL}) acting as discrete base kernels in \emph{Python}.\footnote{The source code can be obtained from \url{https://ls11-www.cs.uni-dortmund.de/people/morris/hashgraphkernel.zip}.}
	We compare our kernels to the GraphHopper kernel (\textsc{GH})~\cite{Fer+2013}, an instance of the graph invariant kernels (\textsc{GI})~\cite{Ors+2015}, and the propagation kernel from~\cite{Neu+2015} which support continuous attributes (\textsc{P2K}). Additionally, we compare our kernel to the Weisfeiler-Lehman subtree kernel and the explicit shortest-path kernel (\textsc{SP}), which only take discrete label information into account, to exemplify the usefulness of using continuous attributes. Since the \textsc{Frankenstein}, \textsc{SyntheticNew}, and \textsc{Synthie} data set do not have discrete labels, we used node degrees as labels instead.
	For \textsc{GI} we used the original \emph{Python} implementation provided by the author of~\cite{Ors+2015}.  The variants of the hash graph kernel are computed on a single core only. For \textsc{GH} and \textsc{P2K} we used the original \emph{Matlab} implementation provided by the authors of~\cite{Fer+2013} and~\cite{Neu+2015}, respectively.
		
	\subsection{Experimental Protocol}
	
	For each kernel, we computed the normalized gram matrix. For explicit kernels we computed the gram matrix via the linear kernel. Note that the running times of the hash graph kernels could be further reduced by employing linear kernel methods. 
	
	We computed the classification accuracies using the \emph{C-SVM} implementation of \emph{LIBSVM}~\cite{Cha+11}, using 10-fold cross validation. The $C$-parameter was selected from $\{10^{-3}, 10^{-2}, \dots, 10^{2},$ $10^{3}\}$ by 10-fold cross validation on the training folds. We repeated each 10-fold cross validation ten times with different random folds and report average accuracies and standard deviations. Since the hash graph kernels and \textsc{P2K} are randomized algorithms we computed each gram matrix ten times and report average classification accuracies and running times. We report running times for \textsc{WL}, \textsc{HGK-WL}, and \textsc{P2K} with $5$ refinement steps.
	
	We fixed the number of iterations of the hash graph kernels for all but  the \textsc{Synthie} data set to 20, since this was sufficient to obtain state-of-the-art classification accuracies. For the \textsc{Synthie} data set we set the number of iterations to 100, which indicates that this data set is harder to classify. The number of refinement/propagation steps for \textsc{WL}, \textsc{HGK-WL}, and \textsc{P2K} was selected from $\{0,\dots,4\}$ using 10-fold cross validation on the training folds only. For the hash graph kernels we centered the attributes dimensionwise to the mean, scaled to unit variance, and used $2$-stable LSH~\cite{Dat+2004} as hash functions to hash attributes to discrete labels. For the \textsc{Enzymes}, the \textsc{Proteins}, \textsc{SyntheticNew}, and \textsc{Synthie} data set we set the interval length to $1$, see \cite{Dat+2004}. Due to the high dimensional sparse attributes of the \textsc{Frankenstein} data set we set the interval length to $100$. 
	For each hash graph kernel we report classification accuracy and running time with and without taking discrete labels into account. For the \textsc{HGK-WL} we propagated label and hashed attribute information separately. In order to speed up the computation we used the same LSH hash function for all attributes in one iteration.
	
	For the graph invariant kernel we used the \textsc{LWL$_{\text{V}}$} variant, which has been reported to perform overall well on all data sets~\cite{Ors+2015}. The implementation is using parallelization to speed up the computation and we set the number of parallel processes to eight. For \textsc{GH} and \textsc{GI} we used the Gaussian RBF kernel to compare node attributes. For all the data sets except \textsc{Frankenstein}, we set the parameter $\gamma$ of the RBF kernel  to 1/(Dimension of attribute vector), see~\cite{Fer+2013, Ors+2015}. For \textsc{Frankenstein}, we set $\gamma = 0.0073$~\cite{Ors+2015}. 
	
	Moreover, in order to study the influence of the number of iterations of the hash graph kernels, we computed classification accuracies and running times of hash kernels with 1 to 50 iterations on the \textsc{Enzymes} data set. Running times were averaged over ten independent runs.
	
	All experiments were conducted on a workstation with an \emph{Intel Core i7-3770@3.40GHz}  and 16 GB of RAM running \emph{Ubuntu 14.04 LTS} with \emph{Python 2.7.6} and \emph{Matlab R2015b}.
	
	\subsection{Results and Discussion}	
		
			\begin{table*}\centering
						\caption{Running times in seconds (Number of iterations for \textsc{HGK-WL} and \textsc{HGK-SP}: 20 (100 for \textsc{Synthie}), Number of refinement steps of \textsc{WL}, \textsc{HGK-WL}, and \textsc{PK}: 5, $^*$--- Kernel uses discrete labels only, $^\dagger$--- Matlab code, $^\ddagger$--- Code is executed in parallel using eight processes), \textsc{OOM}--- Out of Memory.}
						\ra{0.8}
				\resizebox{0.72\textwidth}{!}{
							\begin{tabular}{@{}lccccccccccc@{}}\toprule
								\multirow{3}{*}{\textbf{Graph Kernel}}&\multicolumn{10}{c}{\textbf{Data Set}}\\\cmidrule{2-11}
								& \multicolumn{2}{c}{\textsc{Enzymes}}  & \multicolumn{2}{c}{\textsc{Frankenstein}}  & \multicolumn{2}{c}{\textsc{Proteins}}  &  \multicolumn{2}{c}{\textsc{SyntheticNew}} & \multicolumn{2}{c}{\textsc{Synthie}}  \\
								& {Cont.} & {Label+Cont.} 	& {Cont.} & {Label+Cont.}	& {Cont.} & {Label+Cont.} 	& {Cont.} & {Label+Cont.} 	& {Cont.} & {Label+Cont.}  
								\\\midrule
								\textsc{WL$^*$} &  \multicolumn{2}{c}{1.30}  &  \multicolumn{2}{c}{25.05} & \multicolumn{2}{c}{4.06}   &  \multicolumn{2}{c}{0.73} &  \multicolumn{2}{c}{1.02}  \\ 
								\textsc{SP$^*$} & \multicolumn{2}{c}{1.46} &  \multicolumn{2}{c}{22.87
								} & \multicolumn{2}{c}{5.86}  &  \multicolumn{2}{c}{3.26} &  \multicolumn{2}{c}{3.69} \\
								\textsc{HGK-SP} & 27.91 & 43.32  & \bf 165.9 & 197.82 & 89.13
								 & 107.09 & 60.74 & 80.63 & 428.13 & 714.38 \\
								\textsc{HGK-WL} & 25.10 & \bf 32.06 & 307.10 & 497.69 & 60.00 & \bf 82.41   &  15.17 & 22.50 & 123.59 & 168.04 \\	
								\textsc{GH$^\dagger$} & -- &  365.82 & -- &  16329.00 & --  & 3396.20  &-- &  275.26   & -- & 348.18 \\
								\textsc{GI$^\ddagger$} & -- & 1748.82  &--  & 26717.25 & -- & 7905.23 & -- & 3814.46  & -- & 5522.96 \\
								\textsc{P2K$^\dagger$} & -- &  43.77  & --  & \textsc{OOM} & -- & 208.09 & -- & \bf 15.12 & -- &\bf  45.34 \\
								\bottomrule
							\end{tabular}}
							\label{run}
					\end{table*}
			\begin{table*}\centering
				
					\caption{Classification accuracies in percent and standard deviations (Number of iterations for \textsc{HGK-WL} and \textsc{HGK-SP}: 20 (100 for \textsc{Synthie}), $^*$--- Kernel uses discrete labels only), \textsc{OOM}--- Out of Memory.}
					\ra{0.8}
			\resizebox{1.0\textwidth}{!}{
		
						\begin{tabular}{@{}lcccccccccc@{}}\toprule
							\multirow{3}{*}{\textbf{Graph Kernel}}&\multicolumn{9}{c}{\textbf{Data Set}}\\\cmidrule{2-11}
							& \multicolumn{2}{c}{\textsc{Enzymes}}  & \multicolumn{2}{c}{\textsc{Frankenstein}}  & \multicolumn{2}{c}{\textsc{Proteins}}  &  \multicolumn{2}{c}{\textsc{SyntheticNew}} & \multicolumn{2}{c}{\textsc{Synthie}} 
							 \\
							& {Cont.} & {Label+Cont.} 	& {Cont.} & {Label+Cont.} 	& {Cont.} & {Label+Cont.} 	& {Cont.} & {Label+Cont.}  & {Cont.} & {Label+Cont.} 
							\\\midrule
							\textsc{WL}$^*$ &  \multicolumn{2}{c}{53.97	(\scriptsize $\pm 1.34)$}  &  \multicolumn{2}{c}{ 73.53
												  (\scriptsize $\pm 0.33 $)} 
							&  \multicolumn{2}{c}{75.02 (\scriptsize $\pm 0.58$)}  &  \multicolumn{2}{c}{\bf 98.57  (\scriptsize $\pm 0.30)$} & \multicolumn{2}{c}{	53.60 (\scriptsize $\pm 0.81)$}  \\ 
							\textsc{SP}$^*$ & \multicolumn{2}{c}{42.88	
							 (\scriptsize $\pm 1.04$)} &  \multicolumn{2}{c}{69.51(\scriptsize $\pm 0.35)$}  
							&  \multicolumn{2}{c}{ 75.71 (\scriptsize $\pm 0.34 $)} 
							& \multicolumn{2}{c}{ 83.30(\scriptsize $\pm1.35)$} &
							 \multicolumn{2}{c}{	53.78 (\scriptsize $\pm 0.62)$} 
							\\
							\textsc{HGK-SP} & 66.73 (\scriptsize $\pm 0.91$) & \bf 71.30  (\scriptsize $\pm 0.86$)  & 65.84 
							  (\scriptsize $\pm 0.32 $) & 70.06 (\scriptsize $\pm 0.32$) &  75.14 (\scriptsize $\pm 0.47 $) & \bf 77.47 (\scriptsize $\pm 0.43$) & 80.55 (\scriptsize $\pm 1.29 $)  & 96.46  \scriptsize $\pm 0.61$) &   86.27 (\scriptsize $\pm 0.72$)& 94.34 (\scriptsize $\pm 0.54 $) \\
							\textsc{HGK-WL} & 63.94 (\scriptsize $\pm 1.11$)&  67.63 (\scriptsize $\pm 0.95 $)  & 73.16 (\scriptsize $\pm 0.34 $) & 73.62
												  (\scriptsize $\pm 0.38$) & 74.88 (\scriptsize $\pm 0.64$) & 76.70 (\scriptsize $\pm 0.41 $) & 97.57 (\scriptsize $\pm 0.42$) & \bf 98.84 (\scriptsize $\pm 0.29$) & 80.25  (\scriptsize $\pm 1.37 $)& \bf  96.75 (\scriptsize $\pm 0.51 $) \\
							\textsc{GH} & -- & 68.80 (\scriptsize $\pm	0.96$) & -- & 68.48	 (\scriptsize $\pm 0.26$) &-- & 72.26 (\scriptsize $\pm 0.34$)  & -- & 85.10 (\scriptsize $\pm 1.04$) & -- & 73.18 (\scriptsize $\pm 0.77$) \\
							\textsc{GI} & -- & \bf 71.70   (\scriptsize $\pm 0.79$)  & -- & \bf 76.31  (\scriptsize $\pm 0.33 $) & -- & 76.88 (\scriptsize $\pm 0.47$)  & -- & 83.07 (\scriptsize $\pm 1.10$)  & -- & 95.75 (\scriptsize $\pm 0.50$) \\
								\textsc{P2K} & -- & 69.22 (\scriptsize $\pm 0.34$)  & -- & \textsc{OOM}  & -- &  73.45  (\scriptsize $\pm 0.48 $)  & -- & 91.70  (\scriptsize $\pm 0.86$)  & -- &  50.15  (\scriptsize $\pm 1.92$)  \\
							\bottomrule
		
						\end{tabular} %
					}
					\label{er}
			\end{table*}
	In the following we answer questions Q1--Q3.
	\begin{itemize}
		\item[\textbf{A1}] 	
		The running times and the classification accuracies are depicted in~\cref{run} and~\cref{er}, respectively.
		
		In terms of classification accuracies \textsc{HGK-WL} achieves state-of-the-art results on the \textsc{Proteins} and the \textsc{SyntheticNew} data set. Notice that the \textsc{WL} kernel, without using attribute information, achieves the same classification accuracy on \textsc{SyntheticNew}. This indicates that on this data set the attributes are only of marginal relevance for classification. A different result is observed for the other data sets. On the \textsc{Synthie} data set \textsc{HGK-WL} achieves the overall best accuracy and is more than 20\% better than \textsc{GH} and 40\% better than \textsc{P2K}. The kernel almost achieves state-of-the art classification accuracy on the \textsc{Frankenstein} data set. Notice that the $\gamma$ parameter of the RBF kernel used in \textsc{GI} and \textsc{GH} was finely tuned.
		
		 \textsc{HGK-SP} achieves state-of-the-art classification accuracy on the \textsc{Enzymes} and \textsc{Proteins} data set and compares favorably on the \textsc{SyntheticNew} data set. On the \textsc{Frankenstein} data set, we observed better classification accuracy than \textsc{GH}. Moreover, it performs also well on the \textsc{Synthie} data set.
		
		In terms of running times both instances of the hash graph kernel framework perform very well. On all data sets \textsc{HGK-WL} obtains running times that are several orders of magnitude faster than \textsc{GH} and \textsc{GI}.

		\item[\textbf{A2}] As \cref{er} shows, the choice of the discrete base kernel has major influence on the classification accuracy for some data sets. On the \textsc{Enzymes} data sets $\textsc{HGK-SP}$ performs very favorably, while $\textsc{HGK-WL}$ achieves higher classification accuracies on the \textsc{Frankenstein} data set. On the \textsc{Proteins} and the \textsc{SyntheticNew} data sets both hash graph kernels achieve similar results. \text{HGK-WL} performs slightly better on the \textsc{Synthie} data set.
		
		\item[\textbf{A3}] \cref{itvsac} illustrates the influence of the number iterations on \textsc{HGK-SP} and \textsc{HGK-WL} on the \textsc{Enzymes} data set. Both plots show that a higher number of iterations leads to better classification accuracies while the running time grows linearly. In case of the \textsc{HGK-SP}, the classification accuracy on the \textsc{Enzymes} data set improves by more than 12\% when using 20 instead of 1 iterations. The improvement on the \textsc{Enzymes} data set is even more substantial for \textsc{HGK-WL}: the classification accuracy improves by more than 16\%. At about 30 and 40 iterations for the \textsc{HGK-SP} and \textsc{HGK-WL}, respectively, the algorithms reach a point of saturation.
	\end{itemize}

							\begin{figure}[h]
									\subfigure[Shortest-path kernel]{
																\includegraphics[width=0.23\textwidth]{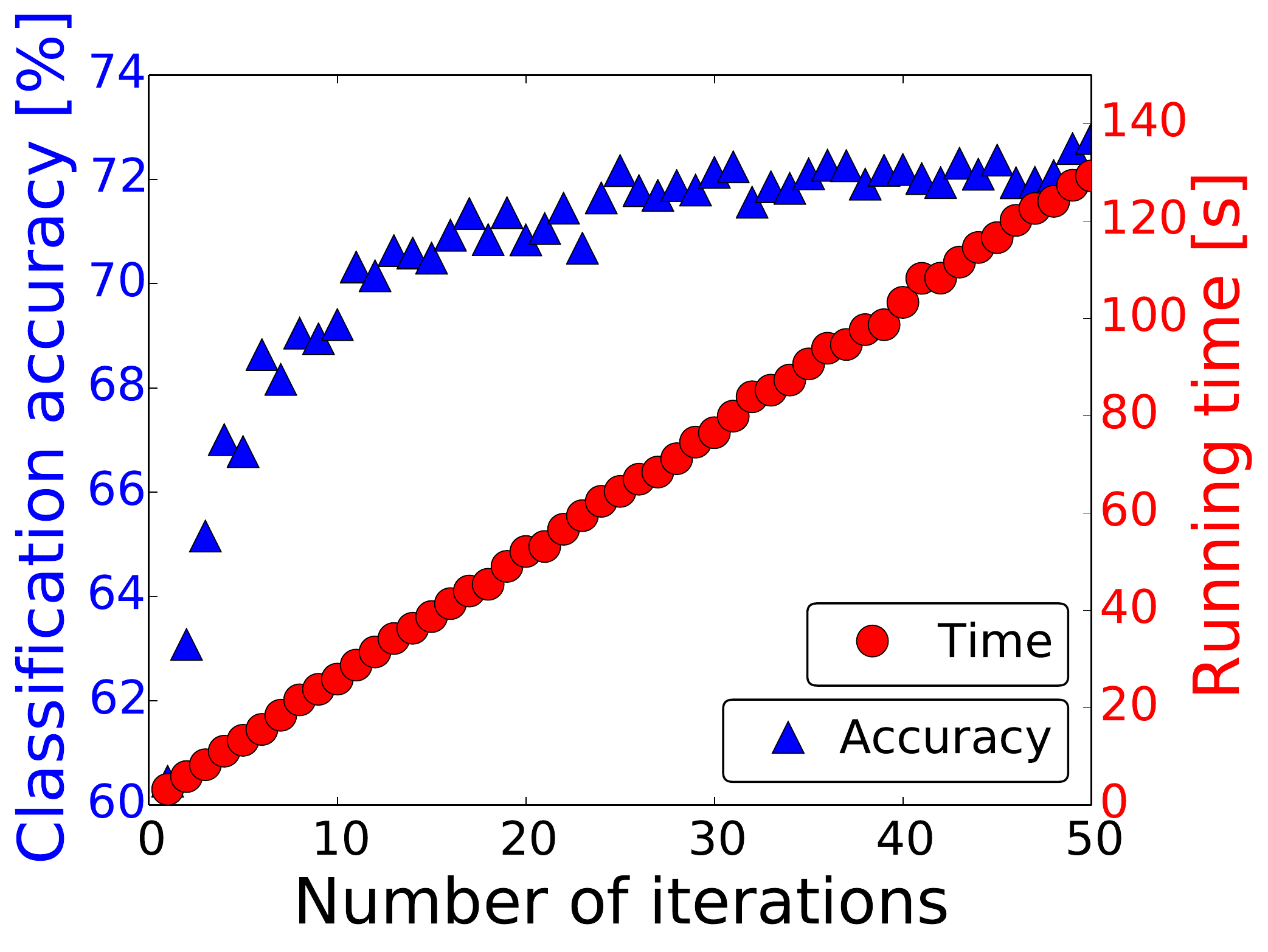}
											}\hfill
											\subfigure[Weisfeiler-Lehman subtree kernel]{
												\includegraphics[width=0.23\textwidth]{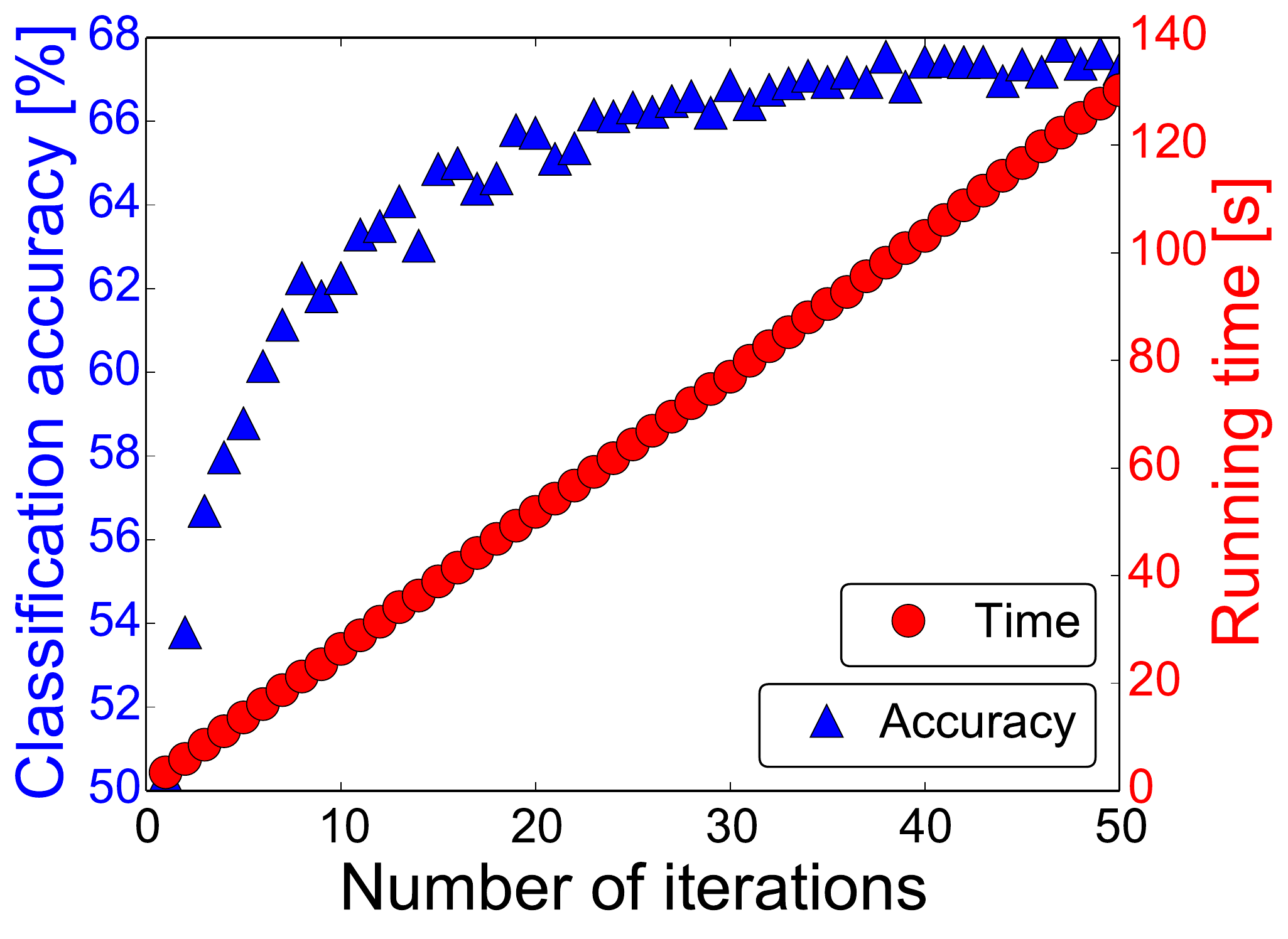}
											}
											
								\caption{Influence of the number of iterations on the classification accuracy for \textsc{HGK-SP} and \textsc{HGK-WL} on the \text{\textsc{Enzymes}} data set.}\label{itvsac}
							\end{figure}

		\section{Conclusion and Future Work}
		
		We have introduced the hash graph kernel framework which allows applying the 
		various existing scalable and well-engi\-neered kernels for graphs with discrete 
		labels to graphs with continuous attributes. 
		The derived kernels outperform other kernels tailored to attributed graphs 
		in terms of running time without sacrificing classification accuracy.

		Moreover, we showed that the hash graph kernel framework approximates implicit variants of the shortest-path and the Weisfeiler-Lehman subtree kernel with an arbitrary small error.

		\section*{Acknowledgement}
		
		This work was supported by the German Science Foundation (DFG) within the Collaborative Research Center SFB 876 ``Providing Information by Resource-Constrained Data Analysis'', project A6 ``Resource-efficient Graph Mining''. We thank Aasa Feragen, Marion Neumann, and Franceso Orsini for providing us with data sets and source code. 

\printbibliography
\end{document}